\documentclass{article}
\usepackage{spconf,amsmath,graphicx}
\usepackage{epsfig}
\usepackage{graphicx}
\usepackage{amsmath}
\usepackage{amssymb}
\usepackage{booktabs}
\usepackage{hyperref}
\usepackage{adjustbox}
\usepackage{amsfonts} 
\usepackage{soul}
\usepackage{amsthm}
\usepackage{mathtools, bm,bbm}
\usepackage{algorithm}
\usepackage{algorithmic}
\usepackage{xcolor}
\usepackage{array,multirow}
\usepackage{tabularx}
\usepackage{color, colortbl}
\definecolor{Gray}{gray}{0.90}
\usepackage{pifont}
\newcommand{\W}{\mathbf W}
\newcommand{\ftheta}{f_{\phi}}
\newcommand{\xx}{\mathbf x}
\newcommand{\XXb}{\mathbb{X}_{\text{base}}}
\newcommand{\XXs}{\mathcal{S}}
\newcommand{\XXq}{\mathcal{Q}}
\newcommand\norm[1]{\left\lVert#1\right\rVert}

\newcommand{\ww}{\mathbf w}
\newcommand{\yy}{\mathbf y}

\newcommand{\g}{\mathbf{g}}
\DeclareMathOperator*{\argmax}{arg\,max}

\newtheorem{prop}{Proposition}

\title{Task Adaptive Feature Transformation for One-Shot Learning}
\name{Imtiaz Masud Ziko$^{1}$, Freddy Lecue$^2$, Ismail Ben Ayed$^3$}
\address{$^1$ Thales Canada, $^2$ JPMorgan Chase, $^3$ \'ETS Montreal,
}
\begin{document}
%
\maketitle
\begin{abstract}
We introduce a simple non-linear embedding adaptation layer, which is fine-tuned on top of fixed pre-trained features for one-shot tasks, improving significantly transductive entropy-based inference for low-shot regimes. Our norm-induced transformation could be understood as a re-parametrization of the feature space to disentangle the representations of different classes in a task specific manner. It focuses on the relevant feature dimensions while hindering the effects of non-relevant dimensions that may cause overfitting in a one-shot setting. We also provide an interpretation of our proposed feature transformation in the basic case of few-shot inference with K-means clustering. Furthermore, we give an interesting bound-optimization link between K-means and entropy minimization. This emphasizes why our feature transformation is useful in the context of entropy minimization. We report comprehensive experiments, which show consistent improvements over a variety of one-shot benchmarks, outperforming recent state-of-the-art methods. 
\end{abstract}
\section{Introduction}
\label{sec:intro}

Deep learning models have achieved impressive success in a breadth of applications. However, these successes mostly rely on learning from huge amounts of annotated data, which requires a time-consuming and expensive process. Deep learning models still have difficulty generalizing to novel classes unseen during training, given only a few labeled instances for these new classes. In this context, few-shot learning research has attracted wide interest recently. For example in a one-shot learning setting, a model is first trained on substantial labeled data over an initial set of classes, commonly called the base classes. Then, supervision is confined to one labeled example per novel class, which is not observed during base training. The model is then fine-tuned on these labeled examples from the novel classes (the support set) and evaluated on the unlabeled samples (the query set). Traditional fine-tuning would result in over-fitting in such low-data regimes. A large body of works investigated few-shot learning via meta-learning strategies, such as the very popular prototypical networks \cite{snell2017prototypical}. Meta-learning creates a set of few-shot tasks (or episodes), with support and query samples that simulate generalization difficulties during testing and train the model to generalize well on these tasks. 

\begin{figure*}[htbp]
\centering
\includegraphics[width=\textwidth]{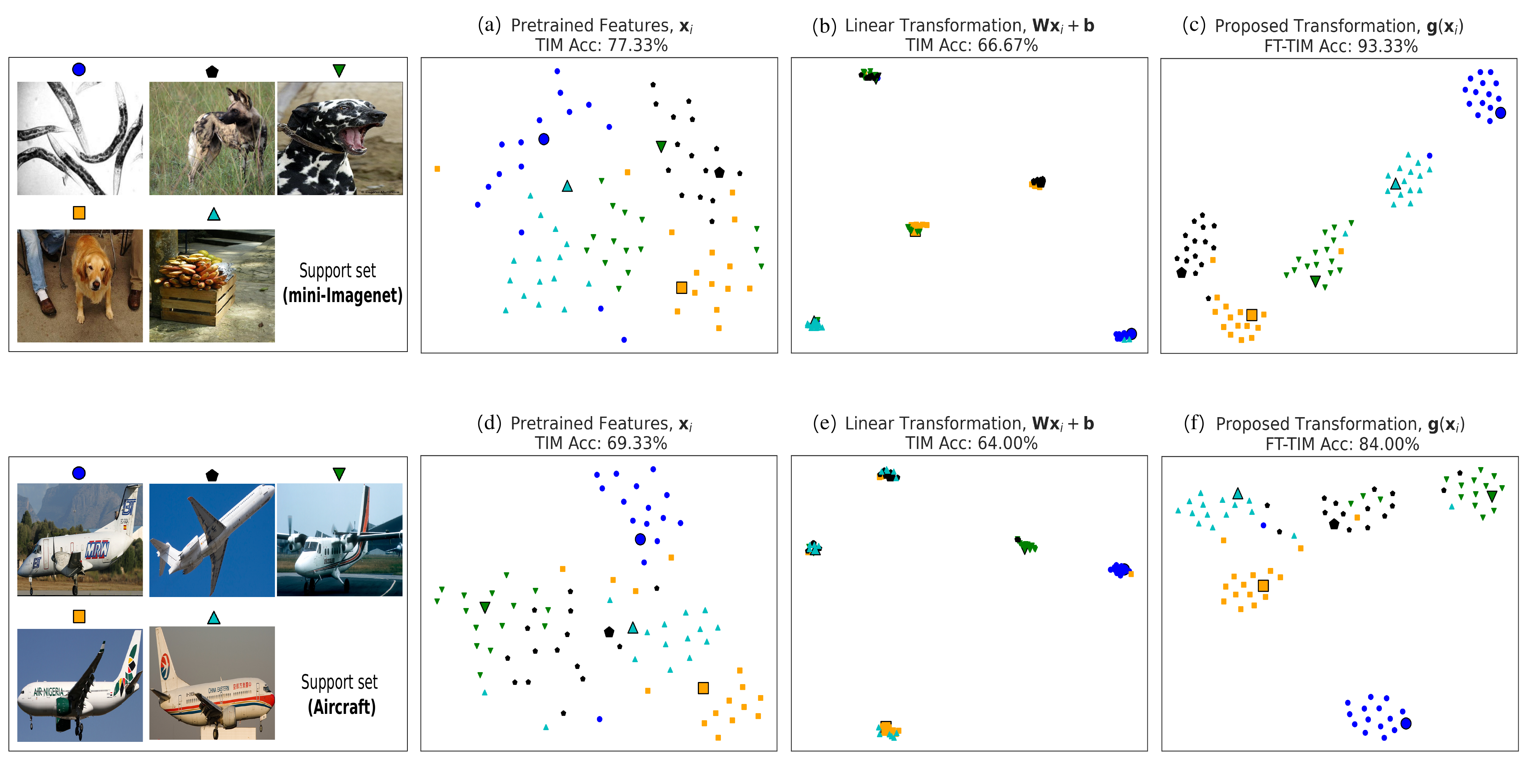}
\caption{TSNE plots depicting the feature space with or without the proposed feature transformation in \eqref{eq:transformation}. The support images of the 1-shot task are provided (leftmost). The bigger markers correspond to the support image in each class.}
\label{fig:tsne}
\end{figure*}

\textbf{Related works}: Our method is in line with recent transductive methods in the few-shot learning literature, e.g. \cite{hong2021reinforced,liu2018learning,Dhillon2020A,Laplacian,boudiaf2020transductive}, among others. Transductive inference performs class predictions jointly for all the unlabeled query samples of the task, rather than one sample at a time as in inductive inference. 
For instance, TPN \cite{liu2018learning} uses label propagation along with episodic training and a specific network architecture; the goal was to learn how to propagate labels from the support to the query samples. CAN-T \cite{can} is another meta-learning based transductive method, which uses attention mechanisms to propagate labels to unlabeled samples. The authors of \cite{Laplacian} proposed a method based on graph clustering, which regularizes the inductive predictions of query samples with a Laplacian term.

Many transductive few-shot methods focused on strategies for fine-tuning a pre-trained model during inference. For instance, the entropy fine-tuning in \cite{Dhillon2020A} re-trains the whole network, performing costly gradient updates over all the parameters during inference. Transductive Information Maximization (TIM) \cite{boudiaf2020transductive} proposes an entropy-based fine-tuning loss, which maximizes the mutual information between the query features and their label predictions for a few-shot task at inference, while minimizing the cross-entropy loss on the support set. However, instead of retraining the whole network, \cite{boudiaf2020transductive} only fine-tunes the softmax layer on top of the fixed pre-trained features. 
This showed substantial improvements over retraining the whole network. In addition to its recent successful use in transductive few-shot classification \cite{boudiaf2020transductive,Dhillon2020A}, it is worth noting that entropy minimization is widely used in semi-supervised learning \cite{grandvalet2005semi,berthelot2019mixmatch}, and has been successfully used recently in unsupervised domain adaptation \cite{liang2020we} and unsupervised representation learning \cite{ym-self-labelling}.

\textbf{Our Contribution}: Fine-tuning the classifier on top of fixed pre-trained features from the base classes may not take full advantage of the expressive power of the task-specific feature space. A standard linear transformation causes overfitting when dealing with limited supervision. In this regard, we propose
a simple yet effective norm-induced feature transformation, which is fine-tuned to emphasize class-specific feature dimensions while hindering the effect of non-relevant dimensions that may cause overfitting in a few-shot setting. Our non-linear transformation could be understood as a 
re-parametrization of the feature space, which disentangles the representations of the different classes in a task-specific manner. While our motivation is conceptually similar to early kernel-based metric-learning methods \cite{fink2005object}, in which non-linear transformations are implicit, our transformation is explicit. We provide an interpretation of our transformation in the basic case of few-shot inference with K-means clustering. Furthermore, we give an interesting bound-optimization link between K-means and entropy minimization. This emphasizes why our feature transformation is useful in the context of entropy minimization, which is widely used in learning, even beyond few-shot classification. We report comprehensive experiments, showing that the proposed transformation could yield consistent improvements over various one-shot benchmarks, outperforming recent state-of-the-art methods.

\section{Task Adaptive Feature Transformation for One-Shot Learning}
In the one-shot setting, we are given a labeled support set $\XXs$ with $C$  novel test classes, where each novel class has one labeled example. The objective is to accurately classify unlabeled unseen query sample set $\XXq$ from these $C$ classes. Let $\ftheta$ denote the embedding function of a deep convolutional neural network, with parameters $\phi$ and $\xx_i=\ftheta(\yy_i) \in \mathbb{R}^d$ is the features of a given sample $\yy_i$. $\ftheta$ is pre-trained from a labeled set $\XXb$, via a standard cross-entropy loss, with base classes that are different from the test classes of $\XXs$ and $\XXq$.

\textbf{Proposed Transformation}:
\label{proposed-transformation}
The proposed feature transformation is performed during fine-tuning, which is derived from minimizing an entropy-based loss function for the target one-shot task. We will detail the entropy-based loss function below, and draw an interesting connection to the basic K-means objective through bound optimization. For now, let us introduce our non-linear transformation, which reads as follows
for each $L_2$-normalized pre-trained feature vector $\xx_i$ of a given target few-shot task:
\begin{align}\label{eq:transformation}
\begin{split}
\g(\xx_i, \W) & = - \frac{1}{2}\left ( \|\xx_i - \ww_1\|^2, \ldots, \|\xx_i-\ww_d\|^2 \right )^T
\end{split}
\end{align}
where $\xx_i$ is the initial feature vector either from the support set $\XXs$ or query set $\XXq$, and superscript $T$ denotes the transpose operator. 
We introduce a learnable transformation matrix $\W = [\ww_1^T\ldots \ww_d^T] \in \mathbb{R}^{d\times d}$, which is updated during the fine-tuning procedure. To understand the effect of our transformation, let us first consider a transductive inference with a basic K-means clustering of the transformed features of the query set. 
This could be done by optimizing the following mixed objective:
\begin{equation}
\label{basic-K-means}
\mathcal{J}(\W, \boldsymbol{\theta}, \boldsymbol{Q}) = 
\sum_{i \in \XXq}\sum_{c=1}^C q_{ic} \|\boldsymbol{\theta}_c - \g(\xx_i, \W)\|^2
\end{equation}
where $\boldsymbol{\theta} = (\boldsymbol{\theta}_c)_{1 \leq c \leq C}$ represent class prototypes, and 
$\boldsymbol{Q}$ is the $|\XXq|$-by-$C$ matrix whose rows are given by binary assignment 
simplex vectors $\boldsymbol{q}_i =  (q_{ic})_{1 \leq c \leq C} \in \{0, 1\}^C$: $q_{ic} = 1$ if sample $\xx_i$ is assigned to class $c$ and $q_{ic} = 0$ otherwise. Alternating iterative minimization of the mixed K-means objective in Eq. \eqref{basic-K-means} with respect to $\W$, $\boldsymbol{\theta}$ and $\boldsymbol{Q}$ could be viewed as a joint task-adaptive metric learning and clustering, and has a clear interpretation of the effect of the proposed transformation. Optimization with respect to transformation parameters $\W$ encourages new features $\g(\xx_i, \W)$ to approach their current class prototypes (or means) $\boldsymbol{\theta}_c$, thereby disentangling the class representations; see the TSNE-plot illustrations in Figure \ref{fig:tsne}. Clearly, given current assignments $q_{ic}^j$ at iteration $j$, the optimal $\boldsymbol{\theta}_c$ minimizing \eqref{basic-K-means} corresponds to the mean of features within class $c$: $\boldsymbol{\theta}_c^j = \frac{\sum_{i \in \XXq} \g(\xx_i, \W)}{|\XXq|}$. Also, given both $q_{ic}^j$ and $\boldsymbol{\theta}_c^j$, it is clear that the objective in  Eq. \eqref{basic-K-means} contains, for each sample, an $L_2$ distance between the transformed feature of the sample and its current-class prototype. Therefore, optimization with respect to $\W$ encourages the transformed feature to align with its current-class prototype.  

It is important to note that the specific norm-induced form of $\mathbf{g}$ that we propose in \eqref{eq:transformation} constrains implicitly the transformation, hindering the effects of non-relevant dimensions that may cause over-fitting. An unconstrained transformation $\mathbf{g}$, such as a neural net for instance, trained jointly with K-means, might yield trivial solutions i.e. bringing all of the transformed features in the same cluster. This difficulty is known in the context of deep clustering \cite{Jabi2021}. {\em In fact, the norm-based form of each component in transformation \eqref{eq:transformation} forces some dimensions to approach zero, when aligning with the prototypes}. Identity, for instance, is not recoverable under this constrained form, unlike an unconstrained neural-net transformation.

Let us now give more details in terms of the TSNE plots in Figure \ref{fig:tsne}. We  randomly sample a 1-shot task with 5 test classes from each of the \textit{mini}ImageNet and Aircraft datasets. The task from \textit{mini}ImageNet contains three fine-grained classes sampled from the generic dog category: `Golden retriever', `Dalmatian', and `African hunting dog'. The other two classes are from generic categories: `Nematode' and `Crate'. The task from the Aircraft dataset contains the fine-grained categories of 5 different airplane models, which are visually quite similar to each other. The leftmost plots refer to the entropy-based loss in \eqref{eq:tim_objective} fine-tuned on top of the initial fixed pre-trained feature vector $\xx_i$ without the feature transformation. In this case, only the classifier weights $\{\boldsymbol{\theta}_c\}_{c=1}^{C}$ are updated, which results in accuracies of 77.3\% for \textit{mini}Imagenet and 69.33\% for the Aircraft. Note that the support images (with bigger markers) are not well separated along with their corresponding query samples in the pre-trained feature space learned from the base classes. Linear transformation for feature transformation during fine-tuning causes overfitting with limited supervision, as can be seen in figures 1(b) and 1(e). In this case, features from different classes are brought in the same cluster and the estimates of
errors could be found from the respective classification accuracy given at the top of each plot. Finally, if we utilize the proposed transformation in Eq. \eqref{eq:transformation} on top of the initial pre-trained features, we achieve a better spread-out task adaptive feature space as shown in the rightmost TSNE plot 1(c) and 1(f), with boosted accuracies of 93.33\% (16\% improvement) and 84.00\% (13\% improvement), respectively for \textit{mini}Imagenet and Aircraft dataset.

\textbf{Entropy-based loss function}:
In our case, the transformation matrix $\W$ is learned by fine-tuning a transductive information maximization (TIM) loss \cite{boudiaf2020transductive}.
TIM loss is a combination of cross-entropy defined over the support set $\XXs$ and a mutual information term, which includes two Shannon entropies: The entropy of posterior predictions (i.e., the softmax outputs of the network) and the entropy of the marginal probabilities of the classes, both defined over the query set $\XXq$:
	\begin{align}\label{eq:tim_objective}
	    \begin{split}
            \mathcal{L}(\W, \boldsymbol{\theta}) = \overbrace{-\frac{\lambda}{|\XXs|} \sum_{i \in \XXs} \sum_{c=1}^C y_{ic} \log (p_{ic})}^{\text{cross-entropy}}\\ 
             \underbrace{- \frac{\alpha}{|\XXq|} \sum_{i \in \XXq}\sum_{c=1}^C p_{ic} \log (p_{ic})}_{\text{conditional entropy}} 
	        &+\underbrace{\sum_{c=1}^C \widehat{p}_{c} \log \widehat{p}_{c}}_{\text{marginal entropy}}
	    \end{split}
	\end{align}
where
\[p_{ic} = \text{s}(\boldsymbol{\theta}_c, \W, \xx_{i}) = \frac{\exp \left (-\frac{\tau}{2}\|\boldsymbol{\theta}_c - \g(\xx_i, \W) \|^2 \right )}{\sum_k \exp \left (-\frac{\tau}{2}\|\boldsymbol{\theta}_k - \g(\xx_i, \W)\|^2 \right )} \]
denotes the softmax probability outputs, $\widehat{p}_c = \frac{1}{|\XXq|} \sum_{i \in \XXq} p_{ic}$ is the marginal probability of class $c$, and $y_{ic} \in \{0, 1\}$ the ground-truth labels for the support samples and $\tau$ is the temperature parameter.
Minimizing the conditional entropy pushes the network probability predictions toward the vertices of the simplex, yielding confident predictions. The marginal entropy term avoids the trivial single-class solutions that might result from conditional entropy minimization. 

\textbf{On the link between entropy and K-means}:
We now change gear and show an interesting bound-optimization link between the conditional entropy in Eq. \eqref{eq:tim_objective} and K-means in Eq. \eqref{basic-K-means}. This link further clarifies why our feature transformation is useful in the context of entropy minimization. In addition to its 
recent successful use in few-shot classification \cite{boudiaf2020transductive,Dhillon2020A}, entropy is widely used in semi-supervised learning \cite{grandvalet2005semi,berthelot2019mixmatch}, and 
has been successfully used recently in unsupervised domain adaptation \cite{liang2020we} and 
unsupervised representation learning \cite{ym-self-labelling}. Therefore, connecting entropy minimization to K-means could provide interesting insights even beyond few-shot classification. To show the link, let us first decompose 
the conditional entropy in \eqref{eq:tim_objective}:
\begin{equation}
\label{conditional-entropy-decomposition}
 \underbrace{\sum_{i, c} \text{s} (\boldsymbol{\theta}_c, \W, \xx_{i} ) \|\boldsymbol{\theta}_c - \g(\xx_i, \W)\|^2}_{\mathcal{H}(\W, \boldsymbol{\theta}): \, \text{Clustering}} +  \underbrace{\sum_i l(\boldsymbol{\theta}, \W, \xx_{i})}_{\text{Prototype} ~\text{dispersion}}
\end{equation}
where $l(\boldsymbol{\theta}, \W, \xx_{i}) = \log \sum_c \exp \left ( -\frac{\tau}{2}\|\boldsymbol{\theta}_c - \g(\xx_i, \W) \|^2 \right )$. 
Minimizing the prototype dispersion encourages large distances between the prototypes and the features of all data points.  
Term $\mathcal{H}(\W, \boldsymbol{\theta})$ in Eq. \eqref{conditional-entropy-decomposition} is closely related to basic K-means \eqref{basic-K-means} from a bound-optimization perspective, although it seems more complex. The following shows that 
optimizing a soft K-means could be viewed as an approximate Majorize-Minimize (MM) algorithm for optimizing $\mathcal{H}(\W, \boldsymbol{\theta})$. 

Given a function $\mathcal{H}(\W, \boldsymbol{\theta})$, the 
general MM paradigm minimizes iteratively a tight upper bound on $\mathcal{H}$:
\begin{align}
\label{aux-function}
 \mathcal{H}(\W, \boldsymbol{\theta}) &\leq \mathcal{A}^j(\W, \boldsymbol{\theta}) ~\forall~\W, \boldsymbol{\theta} \nonumber \\
 \mathcal{H}(\W^j, \boldsymbol{\theta}^j) &= \mathcal{A}^j(\W^j, \boldsymbol{\theta}^j)
\end{align}
where $j$ is the current iteration index. An upper bound satisfying the tightness condition in \eqref{aux-function}  is often referred to as {\em auxiliary function} of the original objective $\mathcal{H}$. 
It is straightforward to verify 
that minimizing $\mathcal{A}^j$ iteratively guarantees the original objective does not increase: 
$\mathcal{H}(\W^{j+1}, \boldsymbol{\theta}^{j+1}) \leq \mathcal{A}^j(\W^{j+1}, \boldsymbol{\theta}^{j+1}) \leq \mathcal{A}^j (\W^j, \boldsymbol{\theta}^j) = \mathcal{H}(\W^j, \boldsymbol{\theta}^j)$.

\begin{algorithm}[tb]
\caption{FT-TIM inference}
\label{alg:algorithm}
\textbf{Input}: Pre-trained encoder $\ftheta$, One-shot task $\{\XXs, \XXq\}$\\
\textbf{TIM Parameters}: Number of Fine-tuning iterations $I$, temperature $\tau$, loss weights $\{\lambda, \alpha\}$,learning rate for $\{\boldsymbol{\theta}_c\}_{c=1}^{C}$ \\
\textbf{FT Parameters}: Learning rate for transformation matrix $\W$, iteration number $I_W$ to start transformation \\
\begin{algorithmic}[1] 
\STATE L2 normalization: $\xx_i = \ftheta(x_i) /\norm{\ftheta(x_i)}_2$, $\forall x_i \in {\XXs \cup \XXq}$
\STATE Initialize $\W = \mathbf{X}_s^T\mathbf{X}_s$, where $\mathbf{X}_s=\{\xx_s\}, \forall \xx_s \in\XXs$
\STATE Initialize $iter = 0$
\WHILE{$iter\leq I$}
\STATE $\xx^{'}_i = \xx_i$
\IF {$iter\geq I_W$}
\STATE Compute transformed feature $\g(\xx_i)$ from \eqref{eq:transformation}
\STATE L2 normalization: $\xx^{'}_i = \frac{g(\xx_i)}{\|g(\xx_i)\|_2}$ .
\ENDIF
\STATE Compute $p_{ic} = \text{softmax}(-\frac{\tau}{2}\|\boldsymbol{\theta}_c - \xx^{'}_{i}\|^2), \forall i, \forall c$
\STATE Compute marginal $\widehat{p}_c = \frac{1}{|\XXq|} \sum_{i \in \XXq} p_{ic}, \forall c$
\STATE Compute Loss in \eqref{eq:tim_objective}
\IF {$iter\geq I_W$}
\STATE Update $\W$
\ENDIF
\STATE Update $\{\boldsymbol{\theta}_c\}_{c=1}^{C}$
\ENDWHILE
\STATE \textbf{return} Query predictions $\widehat{y}_i = \argmax_{c} p_{ic}, \forall i \in \XXq$
\end{algorithmic}
\end{algorithm}

\begin{prop}
$\mathcal{H}(\W, \boldsymbol{\theta})$ is upper bounded by the following soft K-means objective for any set of soft simplex assignment variables $\boldsymbol{q}_i =  (q_{ic})_{1 \leq c \leq C} \in [0, 1]^C$, $i \in \XXq$:
\begin{equation}
\label{upper-bound-soft-keans}
\mathcal{H}(\W, \boldsymbol{\theta}) \leq \mathcal{J}(\W, \boldsymbol{\theta}, \boldsymbol{Q}) + \frac{\tau}{2} \sum_i \boldsymbol{q}_i^T \log \boldsymbol{q}_i
\end{equation}
Furthermore, given parameters $\W^j$ and prototype $\boldsymbol{\theta}^j = (\boldsymbol{\theta}_c^j)_{1 \leq c \leq C}$ at iteration $j$, choosing specific expressions $q_{ic} = \text{s} (\boldsymbol{\theta}_c^j, \W^j, \xx_{i} )$ in 
 upper bound \eqref{upper-bound-soft-keans} yields an approximate auxiliary function on $\mathcal{H}(\W, \boldsymbol{\theta})$ when $\tau$ is small ($\tau \rightarrow 0$).  
\end{prop}
\begin{proof}
The upper bound in \eqref{upper-bound-soft-keans} is convex w.r.t $\boldsymbol{Q}$ as it is the sum of linear and convex functions. Solving the KKT conditions for minimizing this bound, s.t. 
simplex constraint on each $\boldsymbol{q}_i$, yields closed-form solutions:
$\tilde{q}_{ic} = \text{s} (\boldsymbol{\theta}_c, \W, \xx_{i})$
The inequality in \eqref{upper-bound-soft-keans} follows directly from plugging these optimal solutions in the upper-bound in \eqref{upper-bound-soft-keans} and using the fact 
that $\tau$ is small ($\tau \rightarrow 0$). Finally, it is straightforward to verify that the specific choice 
$q_{ic} = \text{s} (\boldsymbol{\theta}_c^j, \W^j, \xx_{i} )$ makes 
the upper bound in \eqref{upper-bound-soft-keans} tight at the current solution and, hence, an auxiliary function, when temperature $\tau \rightarrow 0$. 
\end{proof}
\begin{table*}[tbp]
\begin{center}
\begin{adjustbox}{max width=\linewidth}
\begin{tabular}{lccccc}
\toprule
\textbf{Methods}& \textbf{Network} &\multicolumn{1}{c}{\textbf{\textit{mini}ImageNet}}&\multicolumn{1}{c}{\textbf{\textit{tiered}ImageNet}}&\multicolumn{1}{c}{\textbf{CUB}}&\multicolumn{1}{c}{\textbf{Aircraft}}\\
\midrule

MAML \cite{Finn2017ModelAgnosticMF} & ResNet-18 & 49.61  &  -& 68.42&- \\
TPN \cite{liu2018learning} &ResNet-12 & 59.46 & - & - & -\\
Entropy-min \cite{Dhillon2020A} &ResNet-12 & 62.35  & 68.36  &-&-\\
DPGN \cite{yang2020dpgn} &ResNet-18 & 66.63 & 70.46 &-&-\\
CAN+T \cite{can} &ResNet-18 & 67.19 & 73.21 &-&-\\
DSN-MR \cite{simon2020adaptive} &ResNet-18 & 64.60 & 67.39 &-&-\\
MetaoptNet \cite{lee2019meta} &ResNet-18 &62.64& 65.99 & - &- \\
LaplacianShot \cite{Laplacian} &ResNet-18 &70.89& 77.60 & 79.93 &- \\
TIM \cite{boudiaf2020transductive} &ResNet-18 &72.77  & 80.80 &82.24&83.06 \\
RAP-LaplacianShot \cite{hong2021reinforced} &ResNet-12 & 74.29 & -& 83.59 & - \\
FT-TIM (ours) &ResNet-18 &\textbf{75.00}  & \textbf{83.45} &\textbf{85.54} & \textbf{84.47} \\
\hline
AWGIM \cite{guo2020attentive} & WRN & 63.12 & 67.69&-&- \\ 
Entropy-min \cite{Dhillon2020A} & WRN & 65.73 & 73.34&-&- \\
SIB \cite{hu2020empirical}& WRN & 70.0 & 70.90 &-&- \\
BD-CSPN \cite{liu2019prototype} & WRN & 70.31 & 78.74 &-&-\\
SIB+E$^3$BM \cite{liu2020ensemble} & WRN & 71.4 & 75.6 &-&- \\
LaplacianShot \cite{Laplacian} & WRN &73.44& 78.80&-&-\\
IFSL \cite{yue2020interventional} & WRN & 73.51 & 83.07 &-&-\\
TIM \cite{boudiaf2020transductive} & WRN &77.8 & 82.1&-&- \\
FT-TIM (ours) & WRN &\textbf{79.22} &\textbf{85.06} &-&- \\
\bottomrule
\end{tabular}
\end{adjustbox}
\end{center}
\caption{Average one-shot accuracy (in \%) for the \textit{standard} benchmark.}
\label{tab:mini-tiered}
\end{table*}

\section{Experiments}
\textbf{Datasets}: We used four one-shot benchmarks, including both the fine-grained classification settings (\textbf{CUB} and \textbf{Aircraft}) and standard one-shot classification setting (\textbf{\textit{mini}ImageNet} and \textbf{\textit{tiered}ImageNet}). \textbf{\textit{mini}ImageNet} is a subset of the larger ILSVRC-12 dataset \cite{ILSVRC15}. We use the standard split of 64 classes for base training, 16 for validation, and 20 for testing. \textbf{\textit{tiered}ImageNet} \cite{ren18fewshotssl} is also a subset of the ILSVRC-12 dataset, but with 608 classes instead. We split the dataset into 351 classes for base training, 97 for validation and 160 for testing. \textbf{CUB} \cite{wah2011caltech} is a fine-grained image classification dataset with 200 categories. We split it into 100 classes for base training, 50 for validation and 50 for testing. \textbf{Aircraft} or FGVCAircraft \cite{maji13fine-grained} is a fine-grained image classification dataset with 100 airplane models. Following the same ratio as CUB, we split classes into 50 base classes for training, 25 validation, and 25 test classes. Images are resized to $84 \times 84$ pixels.


\textbf{Implementation Details}: The results of the proposed FT-TIM is reproduced and evaluated in the same settings as in \cite{Laplacian,boudiaf2020transductive} for fair comparisons. The network models are trained with cross-entropy loss on the base classes. 
We utilize the same publicly available pre-trained models of \cite{Laplacian,boudiaf2020transductive} for \textit{mini}ImageNet, \textit{tiered}ImageNet, and CUB. For the Aircraft dataset, we train the model according to the same protocol. The evaluation is done on two different setups of 5-way one-shot benchmark: 1) \textit{Standard one-shot benchmark}, 15 samples per class in the query set for each task, and the average accuracy over query sets are reported. 2) \textit{Semi-supervised one-shot benchmark}, where we treat 15 samples per class as the additional unlabeled samples along with the support set and report the accuracy on a separate held out test set containing 5 test samples from each class. In this setup, we compare the results with and without the proposed task adaptive feature transformation while fine-tuning entropy loss (TIM) \cite{boudiaf2020transductive} in \eqref{eq:tim_objective}. The average accuracy over 600 one-shot tasks are reported. In case of \textbf{FT-TIM}, transformation matrix $\W$ is fine-tuned with $0.01$ learning rate, starting from the 200th fine-tuning iteration, which we decide from mini-Imagenet validation set accuracy. The feature transformation weights $\W$ is initialized from the cosine similarity matrix formed with the $L_2$ normalized initial pre trained support set features. 
\begin{table*}[htbp]
\begin{center}
\begin{adjustbox}{max width=\textwidth}
\begin{tabular}{lcccc}
\toprule
\textbf{Methods} &\textbf{\textit{mini}ImageNet}&\textbf{\textit{tiered}ImageNet}&\textbf{CUB}\\
\midrule
TIM \cite{boudiaf2020transductive} &72.40  & 80.70 &82.60 \\
FT-TIM (ours) &\textbf{73.96} & \textbf{83.45} &\textbf{84.69}\\
\bottomrule
\end{tabular}
\end{adjustbox}
\end{center}
\caption{Average accuracy (in \%) for the \textit{semi-supervised} one-shot learning setup. The best results are highlighted in bold font.}
\label{tab:mini-tiered-test}
\end{table*}
\subsection{Results}

The results of the general one-shot classification are highlighted in Table \ref{tab:mini-tiered}. It can be observed that in each of the datasets and network models, the proposed \textbf{FT-TIM} which includes the proposed feature transformation consistently improves the 1-shot accuracy by $~1-3$\% in comparison to the baseline TIM \cite{boudiaf2020transductive} without the proposed transformation. Note that, the proposed FT-TIM also outperforms the other recent transductive methods such as ICIR \cite{wang2021trust} and RAP-LaplacianShot \cite{hong2021reinforced} by simply fine-tuning the feature transformation during evaluation. The similar consistent improvement is also reflected in the case of fine-grained classification on both of the \textbf{CUB} and \textbf{Aircraft} datasets in Table \ref{tab:mini-tiered}. These results clearly demonstrate that the proposed feature transformation can bring out the expressive power of the task adaptive feature space in one-shot learning. 

We again evaluate the efficacy of the proposed feature transformation in semi-supervised one-shot tasks, where additional unlabeled samples are provided along with the one-shot labeled data per novel class. The transformation weights and the classifier weights are updated during the fine-tuning with the labeled data and the additional unlabeled data in the one-shot task. Finally, the inference is performed on a separate held out test set. To observe the benefit of plugging the proposed feature transformation during fine-tuning entropy based loss, we compare the proposed FT-TIM with respect to baseline TIM  \cite{boudiaf2020transductive} without the proposed transformation. From the results in Table \ref{tab:mini-tiered-test}, we can observe that consistent improvements are achieved by FT-TIM across different datasets, number of shots, and network models. These results indicate that the proposed transformation layer, while fine-tuned on top of pre-trained features jointly with the classifier, helps to disentangle the representations of different classes in a task-specific manner.

\section{Conclusion}

In this paper, we present a simple yet effective feature transformation layer, which brings consistent improvements in transductive one-shot learning while fine-tuned on top of pre-trained features. 
The proposed transformation takes full advantage of the expressive power of the task-specific feature space. It could be understood as a re-parametrization of the feature space, which disentangles the representations of different classes in a task-specific manner. We further provided an interpretation of our transformation in the basic case of few-shot inference with K-means clustering, along with an interesting bound-optimization link between K-means and entropy minimization. This emphasizes why our feature transformation is useful in the context of entropy minimization, which is widely used in learning.

\bibliographystyle{IEEEbib}
\bibliography{readings}

\end{document}